\documentclass[11pt]{article} %

\usepackage{amsmath,amssymb,natbib,graphicx,fullpage}
\usepackage[affil-sl]{authblk}

\usepackage{mkolar_definitions}
\usepackage{tikz}
\usetikzlibrary{decorations.pathreplacing}
\usetikzlibrary{arrows.meta}

\bibpunct{(}{)}{;}{a}{,}{,}

\newcommand\veps{\varepsilon}

\newcommand{\KL}{\operatorname{KL}}
\newcommand{\landmark}{\operatorname{lm}}
\newcommand{\direct}{\operatorname{direct}}
\newcommand{\opt}{\operatorname{opt}}

\newcommand{\Dcontrast}{\Dcal_{\operatorname{contrast}}}
\newcommand{\supp}{\operatorname{supp}}
\newcommand{\var}{\operatorname{var}}

\makeatletter
\newtheorem*{rep@theorem}{\rep@title}
\newcommand{\newreptheorem}[2]{%
\newenvironment{rep#1}[1]{%
 \def\rep@title{#2 \ref{##1}}%
 \begin{rep@theorem}}%
 {\end{rep@theorem}}}
\makeatother

\newreptheorem{theorem}{Theorem}
\newreptheorem{lemma}{Lemma}

\title{Contrastive learning, multi-view redundancy, and linear models}
\author[1]{Christopher Tosh}
\author[2]{Akshay Krishnamurthy}
\author[1,3]{Daniel Hsu}
\affil[1]{Data Science Institute, Columbia University, New York, NY 10027}
\affil[2]{Microsoft Research, New York, NY 10011}
\affil[3]{Department of Computer Science, Columbia University, New York, NY 10027}

\begin{document}

\maketitle
{\def\thefootnote{}
\footnotetext{E-mail:
  \texttt{c.tosh@columbia.edu},
  \texttt{akshaykr@microsoft.com},
  \texttt{djhsu@cs.columbia.edu}}

\begin{abstract}%
  Self-supervised learning is an empirically successful approach to unsupervised
learning based on creating artificial supervised learning problems. A popular
self-supervised approach to representation learning is contrastive learning,
which leverages naturally occurring pairs of similar and dissimilar data
points, or multiple views of the same data. This work provides a theoretical
analysis of contrastive learning in the multi-view setting, where two views of
each datum are available. The main result is that linear functions of the
learned representations are nearly optimal on downstream prediction tasks
whenever the two views provide redundant information about the label.
\end{abstract}

\section{Introduction}
\label{section:intro}

Self-supervised learning has emerged as a popular and empirically successful class of methods for learning useful representations from unlabeled data.
Broadly speaking, self-supervised learning refers to techniques that take advantage of naturally occurring structure in unlabeled data to create artificial supervised learning problems and then solves these problems using machine learning methods such as deep learning. The hope is that in solving these problems, a learning algorithm will also create internal representations for data that are useful for other downstream learning tasks.  Self-supervised techniques include de-noising autoencoders~\citep{VLBM08}, image inpainting~\citep{PKDDE16}, and the focus of this work, contrastive learning~\citep{HCL06,oord2018representation,LL18,hjelm2018learning,AKKPS19,bachman2019learning,tian2019contrastive,TKH20,CKNH20}.

A common theme among many of these self-supervised representation learning works is the exploitation of naturally occurring similar points, or multiple views of the same data points. To train a de-noising autoencoder, one first creates an alternate ``view'' of data points by corrupting them with added noise and then trains the autoencoder to reconstruct the original. Image inpainting removes patches of images and trains models to reconstruct the original image. Contrastive learning trains models to distinguish naturally occurring similar pairs of points, such as neighboring sentences~\citep{LL18} or randomly cropped and blurred versions of the same image~\citep{CKNH20}, from random pairs of points.

However, exploiting multiple views of data for representation learning is not a new technique. Canonical correlation analysis (CCA)~\citep{H36} is a classical (unsupervised) technique that finds the linear transformation that aligns two views of data so that the resulting coordinates are uncorrelated. A fascinating line of work \citep{ando2005framework, KF07, AZ07, FJKZ09} investigated the quality of representations produced by CCA (and related linear methods) for downstream regression problems. Most relevant to the current work, \citet{KF07} and \citet{FJKZ09} demonstrated that linear regression with the CCA representation in a low-dimensional space will have near-optimal performance relative to the best linear function of the original representation when there is some redundancy among the two views, i.e., whenever the best linear prediction of the label on each individual view is nearly as good as the best linear prediction of the label when both views are used together.

In this work, we examine contrastive learning from the perspective of multi-view redundancy, analogously to the CCA analysis of \citet{KF07} and \citet{FJKZ09}. We show that when there is some redundancy between the two views on the label, contrastive learning leads to representations such that \emph{linear} functions of these representations are competitive with the (possibly \emph{non-linear}) Bayes optimal predictor of the label.
Our analysis is rather general, and we give bounds on the dimensionality of the representations that is sufficient to lead to good performance in the downstream prediction task.
We consider two specific representations based on contrastive learning.
The first of these is a general-purpose construction that uses the ``landmark embedding'' technique of~\citet{TKH20}.
The second representation is formed by solving a particular bivariate optimization problem.
In both cases, we show that we can use low-dimensional representations and still achieve near-optimal downstream performance with linear methods.
We instantiate our results in some simple latent variable models for illustration.

\subsection{Overview of results}

In the multi-view setting, data points are represented as triples of random variables $(X,Z, Y)$ where $X$ and $Z$ represent two views of the data and $Y$ is some label or regression value to be predicted. 
``Views'' should be interpreted liberally here. For example, they could correspond to the first and second halves of a document or to two different distortions of the same image. However, the main property that we will require of our views is that they share redundant information with respect to predicting $Y$. That is, predicting $Y$ from $X$ or $Z$ individually should be nearly as accurate as predicting $Y$ from $X$ and $Z$ together.

When $X$ and $Z$ do satisfy this redundancy property, we show that there is a surprisingly effective prediction strategy: given $X$, first try to infer $Z$ and then predict $Y$ based only on the inferred $Z$. Specifically, we prove the following lemma.

\begin{figure}
\begin{minipage}{0.36\textwidth}
\centering
\begin{tikzpicture}
\draw (-0.7, 1.0) rectangle (4.7, -0.8);
\draw(0,0) circle [radius=0.4] node {$X$};
\draw(2,0) circle [radius=0.4] node {$\hat{Z}$};
\draw(4,0) circle [radius=0.4] node {$\hat{Y}$};
\draw[-{Latex[length=1mm,width=2mm]}] (0.4,0) -- (1.6,0);
\draw[-{Latex[length=1mm,width=2mm]}] (2.4,0) -- (3.6,0);
\draw (1.0,-0.35) node {infer};
\draw (3.0,-0.35) node {predict};
\draw (4.3, 0.75) node {$\mu(x)$};
\end{tikzpicture}\\
\begin{tikzpicture}
\draw (-0.7, 1.05) rectangle (3.7, -1.05);
\draw(0,0.5) circle [radius=0.4] node {$X$};
\draw(0,-0.5) circle [radius=0.4] node {$Z$};
\draw(3,0) circle [radius=0.4] node {$\hat{Y}$};
\draw[-{Latex[length=2mm,width=2mm]}] (0.4,0.5) -- (2.6,0);
\draw[-{Latex[length=2mm,width=2mm]}] (0.4,-0.5) -- (2.6,0);
\draw (1.2,0.0) node {predict};
\draw (3.0, 0.8) node {Bayes};
\end{tikzpicture} \\
Multi-view prediction strategies
\end{minipage}
\begin{minipage}{0.63\textwidth}
  \centering
\begin{tikzpicture}
\draw (-0.5,1.5) ellipse (1.25cm and 0.75cm);
\draw[text width=2cm,align=center] (-0.5,1.5) node {unlabeled data $(x,z)$};
\draw (1.5,2) rectangle (3.5,1);
\draw[text width=2cm,align=center] (2.5,1.5) node {contrastive learning};
\draw[-{Latex[length=1mm,width=2mm]}] (0.75,1.5) -- (1.5,1.5);
\draw[-{Latex[length=1mm,width=2mm]}] (3.5,1.5) -- (4.75,1);
\draw[text width=2cm,align=center] (4.1,1.5) node {$\varphi$};
\draw (2.5,0) ellipse (1.25cm and 0.75cm);
\draw[text width=2cm,align=center] (2.5,0) node {labeled data $(x,z,y)$};
\draw[-{Latex[length=1mm,width=2mm]}] (3.75,0) -- (4.75,0.75);
\draw (4.75,1.35) rectangle (7.75,0.4);
\draw[text width=3cm,align=center] (6.25,0.85) node {linear prediction $x \mapsto w^\top \varphi(x)$};
\end{tikzpicture} \\
Semi-supervised learning setting
\end{minipage}
\caption{Left: Two prediction strategies in the multi-view setup. When
  the views are redundant, $\mu(x)$ (top) is competitive with the Bayes
  optimal predictor (bottom). Right: our semi-supervised setting. The
  representation $\varphi$ is trained in an unsupervised fashion using
  the two views $(x,z)$. Then a linear predictor of features $\varphi$
  is trained using labeled data. }
  \label{fig:main}
\end{figure}
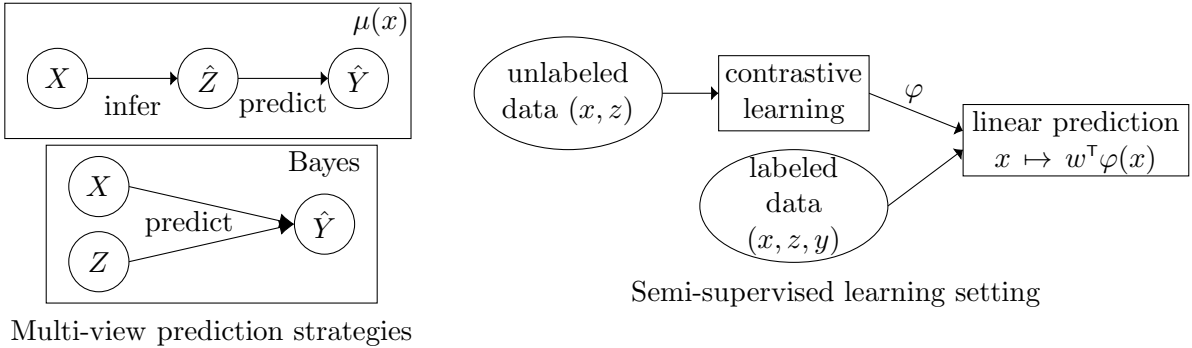

\begin{replemma}{lemma:mu_X-approximation}[Restated]
If $X,Z, Y$ are random variables, then 
\[\EE \left[( \EE[\EE[Y \mid Z] \mid X]  - \EE[Y \mid X, Z])^2 \right] \ \leq \  \veps_X + 2\sqrt{\veps_X\veps_Z} + \veps_Z \]
where $\veps_W = \EE \left[ \left( \EE[Y \mid W] - \EE[Y \mid X, Z] \right)^2 \right]$ for each $W \in \{ X,Z \}$.
\end{replemma}

The strategy of Lemma~\ref{lemma:mu_X-approximation}, illustrated in Figure~\ref{fig:main} (left), is reminiscent of the information bottleneck method~\citep{TPB99}, in which predicting $Y$ from $X$ is done by first compressing $X$ to a ``smaller'' representation $\hat{X}$ and then predicting $Y$ using $\hat{X}$. In our case, the separate view $Z$ acts as a natural intermediate target instead of $\hat{X}$. As it will turn out, Lemma~\ref{lemma:mu_X-approximation} is the basis of all of the following results.

It is also worth mentioning that the quantities $\veps_X$ and $\veps_Z$ in Lemma~\ref{lemma:mu_X-approximation} can be bounded as functions of the conditional mutual information of $X$, $Y$, and $Z$ \citep{tao2006szemer, ahlswede2007final, wu2011functional}. Specifically, \citet{wu2011functional} demonstrated that when $Y \in [-1,1]$, we have
\begin{align*}
\EE \left[ \left( \EE[Y \mid X] - \EE[Y \mid X, Z] \right)^2 \right] \ &\leq \ \frac{1}{2} I(Y;Z \mid X) 
\end{align*}
where $I(Y; Z \mid X)$ is the mutual information of $Y$ and $Z$ conditioned on $X$. A symmetric inequality also holds when $X$ and $Z$ are swapped. Combining this result with Lemma~\ref{lemma:mu_X-approximation}, the performance of the strategy that infers $Z$ from $X$ and then predicts $Y$ based only on this inferred $Z$ can also be bounded as a function of $I(Y; Z \mid X)$ and $I(Y; X \mid Z)$.

We are primarily concerned with the setting where we have lots of unlabeled data, i.e., $(X, Z)$ pairs, and rather less labeled data $(X,Z,Y)$. In such situations, one natural strategy is to use the unlabeled data to learn a representation of $X$ (or $Z$ or $(X,Z)$), and then use the small collection of labeled data to learn a simple function, like a linear predictor, on top of this new representation.
The setting is displayed in Figure~\ref{fig:main} (right).
In this work, we will look at the specific representation learning algorithm posed by~\citet{TKH20}, which is a type of contrastive learning algorithm.

The approach of \citet{TKH20} is to learn a function $f$ that distinguishes between true data points $(X,Z)$, and fabricated data points $(X, \tilde{Z})$, where $X$ and $\tilde{Z}$ come from independently sampled data points $(X,Z)$ and $(\tilde X, \tilde Z)$. The idea is that such a function $f$ will learn enough about the relationship between $X$ and $Z$ to allow us to predict the label $Y$ from $X$ through $Z$ as in Lemma~\ref{lemma:mu_X-approximation}.

In Section~\ref{section:landmark-representations}, we show that one can extract an embedding of a view $X$ from the learned function $f$ such that linear functions on top of this embedding are competitive with the best predictor of $Y$ from $X$. Moreover, this embedding will be the same one proposed by \citet{TKH20} which uses landmark views $Z_1, \ldots, Z_m$ that are i.i.d.\ copies of $Z$ and embeds a point $x$ with the prediction values $f(x,Z_1), \ldots, f(x,Z_m)$.

\begin{reptheorem}{theorem:landmark-embedding-error}[Restated]
  Given a solution $f^\star \colon \Xcal \times \Zcal \to \RR$ to the contrastive learning problem and embedding points $Z_1, \dotsc, Z_m$ sampled i.i.d.\ from the marginal distribution of $Z$, the landmark embedding $\varphi^\star \colon \Xcal \to \RR$ based on $f^\star$ and $Z_1,\dotsc,Z_m$ defined in Section~\ref{section:landmark-representations}
  satisfies, with high probability,
  \begin{align*}
  \min_{w \in \RR^m} \EE \left[ \left( w^\top \varphi^\star(X) - \EE[Y \mid X, Z] \right)^2 \right] 
  \ & \leq \ \veps_X + 2\sqrt{\veps_X\veps_Z} + \veps_Z  + O_m( 1/m ) .
  \end{align*}
\end{reptheorem}
To avoid clutter, we have used big-$O_m$ in the above statement to suppress all factors that do not depend on $m$, including logarithmic factors in the failure probability and other quantities that depend on the distribution of $X$ and $Z$. The full statement is provided in Section~\ref{section:landmark-representations}.

While the results of Section~\ref{section:landmark-representations} demonstrate that contrastive learning can lead to useful representations in the presence of redundancy, the landmark embedding technique is not reflective of what is done in practice. In practice, self-supervised representation learning algorithms typically optimize embedding functions directly~\citep{HCL06, CKNH20}. To address this, in Section~\ref{section:direct-embeddings}, we investigate the strategy of trying to learn the embedding functions directly. That is, we look at the \emph{bivariate architecture} setting where we learn $\RR^m$-valued functions $\eta, \psi$ such that $\eta(x)^\top \psi(z)$ distinguishes between the real and fake data points. As we will see, the benefit of this approach is that when there exist pairs of accurate functions, $\eta$ (and also $\psi$) allow us to do useful linear predictions. However, it is unclear \emph{a priori} how large the output dimension of $\eta$ and $\psi$ needs to be to achieve this.

As a first step towards understanding this dimensionality question, we consider the setting where there is some hidden variable $H$ that renders the two views $X$ and $Z$ conditionally independent. 
Note that there is always a trivial random variable to achieve this, namely $H=(X,Z)$.
However, we show that when the hidden variable obeys a nicer structure, the dimensionality of the embedding can be drastically improved.
Specifically, we show the following.
\begin{itemize}
	\item[(a)] When the hidden variable $H$ takes values in a finite set, the cardinality of this set is an upper bound on the dimensionality needed for an \emph{exact} embedding.
	\item[(b)] In the general setting, there exist approximate embeddings where the approximation factor decreases at a rate of $O_m(1/m)$, where the big-$O_m$ notation suppresses dependence on a particular variance quantity of the hidden variable structure.
\end{itemize}
Importantly, there is no assumption that the hidden variable structure is known. Rather, our results imply that solving the bivariate contrastive learning problem automatically recovers embeddings whose performance can be bounded by factors that depend on the underlying hidden variables. 

Finally, in Section~\ref{section:error-analysis}, we analyze how errors in optimizing the contrastive objectives propagate to the performance of these representations on downstream linear prediction tasks. We investigate this error propagation for both the landmark and direct embeddings, and we show that the downstream prediction risk has a smooth relationship with the excess contrastive loss.

Along the way, we illustrate these results with simplified running examples of a topic model~\citep{blei2003latent} and a Gaussian latent variable model.
However, our results are applicable to other multi-view settings, including co-training, certain mixture models, hidden Markov models, and phylogenetic tree models~\citep[e.g.,][]{blum1998combining,dasgupta2002pac,mossel2005learning,chaudhuri2009multi,allman2009identifiability,anandkumar2012method}.

Additionally, in the appendix we consider the transfer learning scenario, where the unlabeled distribution used for representation learning is not the same as the test distribution used for downstream prediction. We show that in certain settings, the landmark embeddings can be ``fine-tuned" to work under the test distribution.

\subsection{Related work}

A number of recent works have sought to theoretically explain the success of contrastive learning specifically, and self-supervised learning more generally. 
\citet{AKKPS19} presented a theoretical treatment of contrastive learning that considered the specific setting of trying to minimize the loss
$L(\phi) \ = \ \EE [\ell ( \phi(X)^\top (\phi(X_+) - \phi(X_-) ) )]$,
where $(X, X_+)$ is a random ``positive'' pair, $(X,X_-)$ is a random ``negative'' pair, and $\ell$ is a binary classification loss such hinge or logistic loss.
They showed that if there is an underlying collection of latent classes and positive examples are generated by draws from the same class, then minimizing the contrastive loss over embedding functions $\phi$ yields good representations for distinguishing latent classes with linear models.

In work concurrent with the present paper, \citet{LLSZ20} considered a self-supervised scheme in which two views $(X, Z)$ are available for each data point, and the representation learning objective is a reconstruction error of $Z$ based on a function of $X$:
$L(\phi) \ = \ \EE\| Z - \phi(X) \|^2$.
(They assume that $Z$ takes values in a suitable normed space.)
They showed that if the two views are approximately independent conditioned on the label, then linear functions of the learned representation are capable of predicting the label.
This approach resembles the representation learning methods of \citet{ando2005framework,AZ07}, as well as methods for learning predictive state representations dynamical systems~\citep{littman2002predictive,hsu2009spectral,langford2009learning,song2010hilbert}.
The self-supervised problem we study is instead a classification problem rather than a (possibly multidimensional output) regression problem.

Most relevant to the current work, \citet{TKH20} also considered the problem of contrastive learning under certain generative assumptions.
Specifically, they showed that when the two views of the data point correspond to random partitions of a document, contrastive learning recovers information related to the underlying topics that generated the document.
The contrastive learning problem they study is also a classification problem rather than a regression problem.

Also related is the use of self-supervised learning for exploration in a model for reinforcement learning called Block MDPs~\citep{du2019provably,misra2020kinematic}.
In these settings, self-supervised learning is used to derive decoders of unobserved latent state from observations.
The analyses in these works apply to cases where exact decoding of the state is possible.
In particular, the method studied by \citet{misra2020kinematic} uses a contrastive learning objective similar to the one we analyze.
In this paper, we study an example that resembles the Block MDP, but our analysis applies more broadly to scenarios where latent variables cannot be perfectly decoded from the observations.

The contrastive estimation technique we study, now known as ``Noise Contrastive Estimation'' \citep[NCE;][]{GH10}, was also theoretically analyzed in other contexts, including density level set estimation~\citep{SHS05,AZL06}, parametric estimation~\citep{GH10,ma2018noise}, and nonlinear ICA~\citep{hyvarinen2016unsupervised,hyvarinen2017nonlinear,hyvarinen2019nonlinear}.
Although the setups in these works do not consider the use of a learned representation in a downstream task, NCE has inspired many empirical works that use the technique in this way.
The primary motivation for NCE given in these works is the relationship between NCE and maximizing mutual information~\citep[e.g.,][]{oord2018representation,hjelm2018learning,bachman2019learning,tian2019contrastive}, and the usefulness of the learned representation is attributed to this connection.
Although this connection also makes an appearance in our work, it is subordinate to multi-view redundancy in our analysis.
\citet{mcallester2020formal} highlight some limitations on measuring mutual information in these contexts, which raises some doubt that this mutual information perspective can solely explain the success of NCE.
Other doubts about the mutual information perspective are raised by~\citet{tschannen2019mutual}.

\section{Contrastive learning}

In this section, we formalize contrastive learning in the multi-view setting and introduce the redundancy assumption that is key to our analysis.

\subsection{Multi-view data distribution and notation}
\label{section:data}

We consider the multi-view setting, in which data points take the form $(x, z, y) \in \Xcal \times \Zcal \times \RR$, for some pair of data spaces $(\Xcal, \Zcal)$. Here $x$ and $z$ refer to the separate views of the data point, and $y$ refers to its label or regression value.
We assume that there is some distribution over $(x,z,y)$ triples, and we denote the corresponding random variables with capital letters $(X,Z,Y)$. For simplicity, we assume these random variables have either (joint) probability mass functions or probability density functions, and denote them by their corresponding letters, e.g. $p_{X,Z,Y}$, $p_X$, etc.
At various points, we will introduce a hidden variable $H$, and there we shall use $p_{H \mid X}$ to denote the conditional distribution of $H$ given $X$.
For instance, $p_{H \mid X}(h \mid x) = \Pr(H=h \mid X=x)$ when $H$ is discrete.
Finally, we use $p \otimes q$ to denote to the product distribution with marginals $p$ and $q$.

Our main interest is in the semi-supervised learning setting, in which we have both \emph{unlabeled data} from $\Xcal \times \Zcal$, typically modeled as i.i.d.\ copies of $(X,Z)$, as well as \emph{labeled data} from $\Xcal \times \Zcal \times \RR$, modeled as i.i.d.\ copies of $(X,Z,Y)$.
In many cases, the unlabeled data are plentiful, whereas the labeled data are very few due to the cost of obtaining labels.
In this setting, we will use contrastive learning on the unlabeled data to learn a representation that ultimately simplifies the downstream supervised learning task which uses the labeled data.
(In particular, the downstream supervised learning will be accomplished using just \emph{linear} predictors that, we prove, are competitive even with non-linear predictors.)

\subsection{Contrastive distribution}

Following \citet{TKH20}, define the \emph{contrastive distribution} $\Dcontrast$ via the following process:
\begin{itemize}
  \item Let $(\tilde X,\tilde Z)$ be an independent copy of $(X,Z)$, so $(X,Z),(\tilde{X}, \tilde{Z}) \iidsim p_{X,Z}$.
	\item Independently toss a fair coin; if heads, output $(X,Z,1)$; otherwise, output $(X,\tilde Z,-1)$.
\end{itemize}
We let $(X_c,Z_c,Y_c) \sim \Dcontrast$.
Note that $Y_c$ has nothing to do with the random variable $Y$; it is simply the outcome of the fair coin in the generative process above.
Therefore, sampling from $\Dcontrast$ can be accomplished using the process described above as long as one can sample from $p_{X,Z}$---the distribution of unlabeled data.
In practice, this process provides a way to create a self-supervised data set of $(x_c,z_c,y_c)$ triples using only unlabeled data.

\subsection{Contrastive learning problem}
\label{section:clp}

The goal of the contrastive learning problem is to find a predictor of $Y_c$ from $(X_c,Z_c)$; that is, predict whether the two views $(X_c,Z_c)$ are from the same data point (i.e., $(X_c,Z_c) \sim p_{X,Z}$), or from two independent data points (i.e., $(X_c,Z_c) \sim p_X \otimes p_Z$).
This is a binary classification problem, and a standard approach for doing this is to minimize the expected logistic loss between the label $Y_c$ and our prediction $f(X_c,Z_c)$.
Formally, we solve the following optimization problem:
\begin{align}
  f^\star & \ \in \ \argmin_{f \colon \Xcal \times \Zcal \to \RR} L_{\landmark}(f) ,
\qquad 
L_{\landmark}(f) \ := \ \EE \left[ \log \left(1+ \exp\left(- Y_c f(X_c,Z_c) \right)  \right)  \right].
\label{eqn:cross-entropy-contrastive}
\end{align}
(The $\landmark$ subscript is for ``landmark,'' which will be  discussed later.)
Note that the optimal solution $f^\star$ to \eqref{eqn:cross-entropy-contrastive} (over all functions from $\Xcal \times \Zcal$ to $\RR$) predicts the pointwise mutual information between two views:
\[ f^\star(x, z) \ = \ \log \frac{p_{X,Z}(x,z)}{p_{X}(x) p_{Z}(z)} . \]
Given $f^\star$, it is easy to compute the density ratio of the joint distribution of $X$ and $Z$ and the product of their marginals:
\[ g^\star(x,z) \ := \ \exp\left( f^\star(x, z) \right) \ = \   \frac{p_{X,Z}(x,z)}{p_{X}(x) p_{Z}(z)} . \]
Of course, in practice, we typically cannot minimize \eqref{eqn:cross-entropy-contrastive} over all (measurable) functions directly.
Instead, we may use an empirical approximation to the objective based on a finite sample, and (attempt to) minimize this approximation over a particular class of functions (e.g., neural networks).
The issues that arise from such discrepancies are important, but a detailed study is beyond the scope of the present work.
In Section~\ref{section:error-analysis}, we analyze how errors introduced due to these concerns can affect performance in the downstream supervised learning task.

\subsection{Redundancy}

We will assume that for the task of predicting $Y$, there is a certain amount of redundancy between $X$ and $Z$. That is, the quantities
\begin{align*}
\veps_X \ &:= \ \EE \left[ \left( \EE[Y \mid X] - \EE[Y \mid X, Z] \right)^2 \right] \quad \text{ and} \quad
\veps_Z \ := \ \EE \left[ \left( \EE[Y \mid Z] - \EE[Y \mid X, Z] \right)^2 \right] 
\end{align*}
are both small.
As our results hold for any values of $\veps_X$ and $\veps_Z$, we opt not to formalize the redundancy assumption at a particular scale. However, our results are most meaningful in those cases where $\veps_X$ and $\veps_Z$ are small.
We also stress that these predictors ($x \mapsto \EE[Y \mid X=x]$, $z \mapsto \EE[Y \mid Z=z]$, $(x,z) \mapsto \EE[Y \mid (X,Z)=(x,z)]$) are \emph{not} assumed to be linear.

Intuitively, when the redundancy assumption holds, one should be able to get a good prediction on $Y$ by first predicting $Z$ from $X$, and then using the resulting information to predict $Y$. This strategy is formalized by the following function (see the left panel of Figure~\ref{fig:main}):
\[ \mu(x) \ := \ \EE[\EE[Y \mid Z] \mid X=x]. \]
The following lemma tells us that this strategy does indeed work.
\begin{lemma}
\label{lemma:mu_X-approximation}
$\EE[ (\mu(X) - \EE[Y \mid X,Z])^2 ] \ \leq \ \veps_X + 2\sqrt{\veps_X\veps_Z} + \veps_Z \ =: \ \veps_\mu$.
\end{lemma}
Due to space constraints, all proofs are deferred to the appendix.

The function $\mu$ is related to our contrastive function $g^\star$ in the following way:
\begin{align*}
  \mu(x) 
  \ = \ \int \EE[Y \mid Z=z] \frac{p_{X,Z}(x,z)}{p_X(x)} \dif z 
  \ = \ \int \EE[Y \mid Z=z] g^\star(x,z) p_Z(z) \dif z .
\end{align*}
Thus, $g^\star$ provides the change-of-measure from the marginal distribution of $Z$ to the conditional distribution of $Z$ given $X=x$.
Note that $g^\star$ depends only on (the distributions of) $X$ and $Z$; it does not depend on $Y$ at all.
Therefore $g^\star$ is useful for all prediction targets $Y$ for which the redundancy assumption holds.

\section{Landmark embedding representations}
\label{section:landmark-representations}

How should one use $g^\star$ to produce a finite dimensional embedding of $x \in \Xcal$?
When $g^\star$ is implemented using a neural network, a common approach is construct a mapping defined by some internal hidden units (e.g., removing the top layer or two of the network).
How well the resulting embedding fares in downstream tasks may depend on details of the implementation, such as the specific architecture and connection weights.

A different, and generic, approach to constructing an embedding from $g^\star$, proposed by \citet{TKH20}, is to use the external behavior of $g^\star$ on a random sample $Z_1, \dotsc, Z_m$ of views, called \emph{landmarks}, and embed according to
\[ \varphi^\star(x) \ := \ \left( g^\star(x, Z_1), \ldots, g^\star(x,Z_m) \right). \]
We assume $Z_1,\dotsc,Z_m$ are taken from i.i.d.~copies $(X_1,Z_1,Y_1),\dotsc,(X_m,Z_m,Y_m)$ of $(X,Z,Y)$.
In practice, these landmarks can be taken from a random sample of $m$ unlabeled data points.
(One can, of course, construct an embedding of $z \in \Zcal$ symmetrically.)

To gain intuition on why this approach is sound, note that if $w \in \RR^m$ satisfies $w_i = \frac{1}{m}\EE[Y_i \mid Z_i]$, then as $m \rightarrow \infty$ we have
\begin{align*}
w^\top \varphi^\star(x) 
\ = \  \frac{1}{m} \sum_{i=1}^m \EE[Y_i \mid Z_i] g^\star(x, Z_i)  
\ \rightarrow \  \int \EE[Y \mid Z=z] g^\star(x,z) p_Z(z) \dif z \ = \ \mu(x).
\end{align*}
Thus, in the limit, $\varphi^\star(x)$ provides a useful representation for downstream \emph{linear} prediction under the redundancy assumption. 
In this section, we analyze the approximation error that arises when we are restricted to a finite number of landmarks. 

\subsection{Landmark embedding error}

The following lemma quantifies the error from using finite-dimensional landmark embeddings.

\begin{lemma}
\label{lemma:landmark-representation-error}
Let $(X_1,Z_1,Y_1),\dotsc,(X_m,Z_m,Y_m),(X,Z,Y)$ be i.i.d., and define $\varphi^\star$ using the landmarks $Z_1,\dotsc,Z_m$.
With probability $1-\delta$, there exists a weight vector $w \in \RR^{m}$ such that
\begin{align*}
\EE[(w^\top \varphi^\star(X) -  \mu(X))^2 \mid Z_1,\dotsc,Z_m] \ \leq \ \veps_{\landmark}, \quad \text{where} \quad \veps_{\landmark} \ := \ \frac{2\var(\EE[Y_1 \mid Z_1 ] g^\star(X,Z_1))}{\lfloor m/\log_2(1/\delta) \rfloor}.
\end{align*}
\end{lemma}

The upshot of Lemma~\ref{lemma:landmark-representation-error} is that to get an embedding that can linearly approximate $\mu$ within squared loss error of $\epsilon$, it suffices to embed using no more than
\[ O\del{ \frac{\var(\EE[Y_1 \mid Z_1 ]g^\star(X,Z_1))}{\epsilon} } \]
landmarks.
Thus, we have the following immediate consequence of Lemma~\ref{lemma:mu_X-approximation} and Lemma~\ref{lemma:landmark-representation-error} (and the AM/GM inequality).

\begin{theorem}
\label{theorem:landmark-embedding-error}
Let $(X_1,Z_1,Y_1),\dotsc,(X_m,Z_m,Y_m),(X,Z,Y)$ be i.i.d., and define $\varphi^\star$ using the landmarks $Z_1,\dotsc,Z_m$.
With probability $1-\delta$, there exists a weight vector $w \in \RR^m$ such that
\begin{equation*}
  \EE[(w^\top \varphi^\star(X) -  \EE[Y\mid X, Z])^2 \mid Z_1,\dotsc,Z_m] \ \leq \ \veps_\mu + 2\sqrt{\veps_\mu \veps_{\landmark}} + \veps_{\landmark}
\end{equation*}
where $\veps_\mu$ is defined in Lemma~\ref{lemma:mu_X-approximation} and $\veps_{\landmark}$ is defined in Lemma~\ref{lemma:landmark-representation-error}.
\end{theorem}

\subsection{Topic model}
\label{section:tm}

We now turn to a simple topic modeling example to illustrate the error bound.

Let $P_1,\dotsc,P_K$ be distributions over a finite vocabulary $\Vcal$.
Each distribution corresponds to a topic, and we assume that their supports are disjoint; this is similar to the setting studied by~\citet{papadimitriou2000latent}.
For simplicity, suppose each document is exactly two tokens long, so that the two views $(X,Z)$ are individual tokens (and $\Xcal = \Zcal = \Vcal$).
We assume these tokens are drawn from a random mixture of $P_1,\dotsc,P_K$, where the mixing weights are drawn from a symmetric Dirichlet distribution with parameter $\alpha>0$~\citep[following the LDA model of][]{blei2003latent}.
Thus the generative model for a single document is:
\begin{itemize}
  \item Draw $\Theta = (\Theta_1,\dotsc,\Theta_K) \sim \operatorname{Dirichlet}(\alpha)$.
  \item Given $\Theta$, draw $X$ and $Z$ independently from the mixture distribution $\sum_{k=1}^K \Theta_k P_k$.
\end{itemize}

\begin{proposition}
  \label{prop:tm-landmark}
  Assume that $Y$ takes values in $[-1,1]$.
  In the topic model setting,
  \begin{align*}
    \veps_{\landmark}
    & \ \leq \
    \begin{cases}
      \displaystyle
      O\left( \log(1/\delta)/m \right)
      & \text{if $\alpha = \Theta(1)$ as $K\to\infty$} ; \\
      \displaystyle
      O\left( K^2\log(1/\delta)/m \right)
      & \text{if $\alpha \leq 1/K$} .
    \end{cases}
  \end{align*}
\end{proposition}
The $\alpha = \Theta(1)$ and $\alpha \leq 1/K$ correspond to the ``non-sparse'' and ``sparse'' regimes of LDA; here, sparsity is considered in an approximate sense~\citep{telgarsky2013dirichlet}.

\subsection{Gaussian model}
\label{section:gaussian}

As another example, we consider a simple multi-view Gaussian latent variable model
\begin{align*}
  H & \sim \Ncal(0,\sigma^2) , &
  X \mid H & \sim \Ncal(H,1) , &
  Z \mid H & \sim \Ncal(H,1) ,
\end{align*}
and we assume $X \independent Z \mid H$.

\begin{proposition}
  \label{prop:gaussian-landmark}
  Assume $Y$ takes values in $[-1,1]$.
  In the Gaussian model setting, for any $\sigma^2>0$,
  \begin{align*}
    \veps_{\landmark}
    \ = \ O\left(\frac{\log(1/\delta)}{m}\right) \cdot \del{ 1 + \frac{\sigma^4}{1+2\sigma^2} } .
  \end{align*}
\end{proposition}
The analysis shows that the variance of $H$ captures the difficulty of obtaining a low-dimensional representation in this model.
This is intuitive, as the density ratio $g^\star(x,z)$ between the joint distribution and product of marginals scales with the variance of $H$.
We will revisit this example when we consider direct embeddings in Section~\ref{section:direct-embeddings}.

\section{Direct embeddings under hidden variable structure}
\label{section:direct-embeddings}

The results of Section~\ref{section:landmark-representations} demonstrate that contrastive learning, coupled with the landmark embedding method, produces representations that are useful for linear prediction. However, in practice, contrastive learning is not coupled with some landmark embedding. Rather, as an alternative to learning the bivariate function $f^\star(\cdot,\cdot)$, practitioners directly optimize $\RR^m$-valued embedding functions and use these for downstream learning tasks. A typical approach is to solve the following optimization problem:
\begin{align}
  (\eta^\star, \psi^\star) & \in \argmin_{\eta, \psi} L_{\direct}((x,z) \mapsto \eta(x)^\top \psi(z)) , &
L_{\direct}(f) & := \EE \left[ \log \left(1+ \tfrac{1}{f(X_c,Z_c)^{Y_c}} \right) \right].
\label{eqn:cross-entropy-direct}
\end{align}
Here, the minimization is over functions $\eta \colon \Xcal \to \RR^m$ and $\psi \colon \Zcal \to \RR^m$ for some embedding dimension $m$.
Note that this is the same as minimizing the cross-entropy loss when our prediction on $(x,z) \in \Xcal \times \Zcal$ is $\log(\eta(x)^\top \psi(z))$. The loss in \eqref{eqn:cross-entropy-direct} bears a resemblance to the contrastive losses proposed by~\citet{HCL06}, but differing in the fact that (a) the embedding functions are allowed to differ in the two views and (b) inner products are used as opposed to distances. 

Of crucial importance here is the size of the embedding dimension $m$ needed to guarantee good performance on downstream linear predictions. In this section, we tackle this problem in the setting where there is some hidden variable $H$ that renders the two views $X$ and $Z$ conditionally independent. Note that, in general, there is always such a random variable (by taking $H = (X,Z)$); however there may be a much more succinct hidden variable structure, and when this is the case, we show that relatively low-dimensional embeddings can achieve good predictive performance.

\subsection{Discrete hidden variables}

When $H$ is a discrete random variable, taking values in some finite set $S$, \citet{TKH20} observed that we may write
\begin{align*}
g^\star(x,z) \ = \ \eta^\star(x)^\top \psi^\star(z)
\end{align*}
where $\eta^\star \colon \Xcal \rightarrow \RR^{|S|}$ and $\psi^\star \colon \Zcal \rightarrow \RR^{|S|}$ satisfy
\begin{align*}
\eta^\star(x) \ = \ (\Pr(H=h \mid X=x))_{h \in S}  \quad \text{ and } \quad
\psi^\star(z) \ = \ \frac{1}{p_Z(z)}(p_{Z \mid H}(z \mid h))_{h \in S}.
\end{align*}
It is not too hard to show that there is a linear function of $\eta^\star$ that reproduces $\mu$. Namely, we may take $w \in \RR^{|S|}$ to be
\[ w \ = \ \EE[ \psi^\star(Z) \EE[Y \mid Z]] \ = \ \int \psi^\star(z) \EE[Y \mid Z = z] p_Z(z) \dif z.  \]
For this choice of $w$, we have
\begin{align*}
w^\top \eta^\star(x) \ = \ \int \eta^\star(x)^\top \psi^\star(z) \EE[Y \mid Z = z] p_Z(z) \dif z  
\ = \ \int \EE[Y \mid Z = z] g^\star(x,z)  p_Z(z) \dif z \ = \ \mu(x).
\end{align*}

Models with discrete hidden variables
include multi-view mixture models~\citep{chaudhuri2009multi,anandkumar2012method}, as well as models with richer hidden variable structure, such as hidden Markov models and phylogenetic trees~\citep{mossel2005learning,allman2009identifiability}.

\subsection{General hidden variables}

In general, there may not be a discrete random variable that makes $X$ and $Z$ conditionally independent. However, there is always some random variable that does satisfy this. How much such a random variable buys us will naturally depend on its structure and relationship with the views $X$ and $Z$. 
We will establish the existence of a low-dimensional embedding via a probabilistic construction. %
However, before we do this, we first verify that approximating $g^\star$ in a certain sense is sufficient to get good predictions.
\begin{lemma}
\label{lemma:direct-embedding-approximation}
For every $\eta \colon \Xcal \rightarrow \RR^m$ and $\psi \colon \Zcal \rightarrow \RR^m$, there exists a $w \in \RR^m$ such that
\[
  \EE[ (w^\top \eta(X) - \mu(X))^2 ]
  \ \leq \
  \EE[ Y^2 ] \cdot \veps_{\direct}(\eta,\psi)
\]
where
\[
  \veps_{\direct}(\eta,\psi) \ : = \ \EE\left[ \left( \eta(X)^\top \psi(\tilde Z) - g^\star(X,\tilde Z) \right)^2 \right]
\]
and $(X,\tilde Z) \sim p_X \otimes p_Z$.
\end{lemma}
Thus, it suffices to find embeddings whose inner product only approximates $g^\star$ under $p_X \otimes p_Z$.

\subsubsection{A probabilistic construction}

We next show the existence of embeddings $\eta, \psi$ that are low-dimensional and that approximate $g^\star$ in the sense of Lemma~\ref{lemma:direct-embedding-approximation}.

\begin{lemma}
\label{lemma:general-direct-embedding}
Let $H$ denote a random variable that renders $X$ and $Z$ conditionally independent. For any $m > 0$, there exist $\eta \colon \Xcal \to \RR^m$ and $\psi \colon \Zcal \to \RR^m$ such that
\begin{align*}
  \EE\left[ (\eta(X)^\top \psi(\tilde Z) - g^\star(X,\tilde Z))^2 \right]  \ \leq \ \frac1m \var\left( \frac{p_{X\mid H}(X \mid \bar H) p_{Z \mid H}(\tilde Z \mid \bar H)}{p_X(X) p_Z(\tilde Z)} \right)
\end{align*}
where $(X,\tilde Z,\bar H) \sim p_X \otimes p_Z \otimes p_H$.
\end{lemma}

Note that while Lemmas~\ref{lemma:direct-embedding-approximation} and~\ref{lemma:general-direct-embedding} guarantee the existence of low-dimensional embeddings that are useful for downstream linear prediction, they do not guarantee that there exist useful minimizers of \eqref{eqn:cross-entropy-direct} for any given dimension. In Section~\ref{section:error-analysis} we discuss this issue further.

\subsubsection{Examples}

We revisit the topic model and Gaussian model examples from Section~\ref{section:tm} and Section~\ref{section:gaussian}.

\begin{proposition}
  \label{prop:tm-prob}
  Assume that $Y$ takes values in $[-1,1]$.
  In the topic model setting, there exists $\eta \colon \Xcal \to \RR^m$ and $\psi \colon \Zcal \to \RR^m$ such that
  \begin{align*}
    \veps_{\direct}(\eta,\psi)
    & \ \leq \
    \begin{cases}
      \displaystyle
      O\left( 1/m \right)
      & \text{if $\alpha = \Theta(1)$ as $K\to\infty$} ; \\
      \displaystyle
      O\left( K^2/m \right)
      & \text{if $\alpha \leq 1/K$} .
    \end{cases}
  \end{align*}
\end{proposition}
This is essentially the same bound as what was obtained in Proposition~\ref{prop:tm-landmark} for the landmark embedding.
We see that even though the hidden variable structure is not discrete, we still obtain bounds that are polynomial in the dimension of the hidden variable.

\begin{proposition}
  \label{prop:gaussian-prob}
  Assume that $Y$ takes values in $[-1,1]$.
  In the Gaussian model setting, if $\sigma^2 < 1/2$, then there exist $\eta \colon \Xcal \to \RR^m$ and $\psi \colon \Zcal \to \RR^m$ such that
  \begin{align*}
    \veps_{\direct}(\eta,\psi)
    & \ \leq \
    \frac1m \cdot \frac{(1+\sigma^2)^2}{\sqrt{1-4\sigma^4}} .
  \end{align*}
\end{proposition}
Here, we note that the existence argument for $\eta$ and $\psi$ requires a stronger condition than what was required for the landmark embedding to work.
This is reflected in the condition $\sigma^2 < 1/2$.

\section{Error analysis}
\label{section:error-analysis}

We now turn to the problem of bounding the error in the representation incurred by lack of data, imprecise optimization, or restricted function classes.
Specifically, we will be interested in the following measure of risk for an embedding function $\phi \colon \Xcal \rightarrow \RR^m$:
\begin{align}
\label{eqn:risk}
R(\phi)  \ &:= \ \inf_{w \in \RR^m} \EE \left[ \left( w^\top \phi(X) - \mu(X)  \right)^2 \right].
\end{align}
To obtain a guarantee on the mean squared error in approximating $\EE[ Y \mid X, Z]$, we can simply appeal to Lemma~\ref{lemma:mu_X-approximation}.

Our goal is to relate the risk of an embedding to the excess loss in terms of either $L_{\landmark}$ from \eqref{eqn:cross-entropy-contrastive} or $L_{\direct}$ from \eqref{eqn:cross-entropy-direct}.
Note that these two loss functions operate on different scales: the minimizer of $L_{\landmark}$ is the log-odds ratio function $\log \circ g^\star$, and the minimizer of $L_{\direct}$ is the odds ratio function $g^\star$.

We first state the error guarantee for the landmark embedding method.
The analysis is a modification of the argument from \citet{TKH20}.

\begin{theorem}
\label{thm:error-landmark-embedding}
  Assume $Y$ takes values in $[-1,1]$, and let $g_{\max} := \sup_{(x,z) \in \supp p_X \otimes p_Z} g^\star(x,z)$.
  Pick any $f \colon \Xcal \times \Xcal \rightarrow \RR$ such that $g := \exp \circ f$ satisfies $\sup_{(x,z) \in \supp p_X \otimes p_Z} g(x,z) \leq g_{\max}$.
  Let $\varphi \colon \Xcal \rightarrow \RR^m$ be the embedding function given by
  \[ \varphi(x) = \left( g(x,Z_1), \dotsc, g(x,Z_m) \right) \]
  where $Z_1, \dotsc, Z_m$ are i.i.d.\ copies of $Z$.
  With probability at least $1-\delta$ (over the realization of $Z_1,\dotsc,Z_m$),
  \[ R(\varphi) \ \leq \ 2 \veps_{\landmark} + 4 (1+g_{\max})^2 \sqrt{2 \veps_{\opt,\landmark} \veps_{\landmark}} + 16(1+g_{\max})^4 \veps_{\opt,\landmark} \]
where $\veps_{\landmark}$ is defined in Lemma~\ref{lemma:landmark-representation-error} and $\veps_{\opt,\landmark}$ is the excess $L_{\landmark}$-loss:
\[ \veps_{\opt,\landmark} \ := \ L_{\landmark}\left( f \right)  - L_{\landmark}\left( \log \circ g^\star \right). \]
\end{theorem}

We now turn to bounding the risk associated with direct embeddings.

\begin{theorem}
\label{thm:error-direct-embedding-realizable}
  Let $g_{\max} := \sup_{(x,z) \in \supp p_X \otimes p_Z} g^\star(x,z)$.
  Pick any embedding functions $\eta \colon \Xcal \to \RR^m$ and $\psi \colon \Zcal \rightarrow \RR^m$ such that $\sup_{(x,z) \in \supp p_X \otimes p_Z} \eta(x)^\top \psi(z) \leq g_{\max}$.
  We have
  \[ R(\eta) \ \leq \ \EE[Y^2] \, (1+g_{\max})^4 \veps_{\opt,\direct}  \]
  where $\veps_{\opt,\direct}$ is the excess $L_{\direct}$-loss:
  \[ \veps_{\opt,\direct} \ := \ L_{\direct}((x,z) \mapsto \eta(x)^\top \psi(z)) - L_{\direct}(g^\star). \]
\end{theorem}

We point out that $\veps_{\opt,\direct}$ is the excess loss relative to the odds ratio $g^\star$, and not necessarily the best $m$-dimensional representation.
Thus, when there are no perfect $m$-dimensional embedding functions $\eta^\star, \psi^\star$ satisfying $g^\star(x,z) = \eta^\star(x)^\top \psi^\star(z)$, the quantity $\veps_{\opt, \direct}$ accounts for both the error due to optimization and the error due to representational non-realizability.
Moreover, as discussed after the statement of Lemma~\ref{lemma:general-direct-embedding}, we do not actually have a handle on the relationship between embedding dimension and the extent to which bivariate architectures can provide approximations of \eqref{eqn:cross-entropy-direct}.
It therefore remains an interesting future direction to provide a more complete picture of the general setting in which the error due to optimization is teased apart from the representational issues arising from the use of finite dimensional embedding functions.
\subsubsection*{Acknowledgements}
We acknowledge support from NSF grants CCF-1740833, DMREF-1534910, IIS-1563785; a Bloomberg Data Science Research Grant; a JP Morgan Faculty Award; and a Sloan Research Fellowship. We also thank the reviewers whose helpful comments led to improvements in this paper.

\bibliography{refs}

\newpage
\appendix

\section{Omitted proofs}
\label{section:proofs}

\subsection{Proof of Lemma~\ref{lemma:mu_X-approximation}}
\label{section:proof-mu_X-approximation}

By the law of total expectation and Jensen's inequality, we have
\begin{align*}
  \EE[ (\mu(X) - \EE[Y \mid X])^2 ]
  & = \EE[ (\EE[\EE[Y \mid Z] \mid X] - \EE[Y \mid X])^2 ] \\
  & = \EE[ (\EE[\EE[Y \mid Z] - \EE[Y \mid X,Z] \mid X])^2 ] \\
  & \leq \EE[ \EE[(\EE[Y \mid Z] - \EE[Y \mid X,Z])^2 \mid X] ] \quad \text{(Jensen's inequality)} \\
  & = \EE[(\EE[Y \mid Z] - \EE[Y \mid X,Z])^2] \\
  & = \veps_Z .
\end{align*}
Using the AM/GM inequality, for any $\lambda>0$,
\begin{align*}
  \lefteqn{ \EE[ (\mu(X) - \EE[Y \mid X,Z])^2 ] } \\
  & = \EE[ (\EE[Y \mid X] - \EE[Y \mid X,Z] + \mu(X) - \EE[Y \mid X])^2 ] \\
  & \leq (1+1/\lambda) \EE[ (\EE[Y \mid X] - \EE[Y \mid X,Z])^2 ] + (1+\lambda) \EE[ (\mu(X) - \EE[Y \mid X])^2 ] \\
  & \leq (1+1/\lambda)\veps_X + (1+\lambda)\veps_Z .
\end{align*}
Optimizing the bound with respect to $\lambda$ gives
\begin{align*}
  \EE[ (\mu(X) - \EE[Y \mid X,Z])^2 ]
  & \leq \veps_X + 2\sqrt{\veps_X\veps_Z} + \veps_Z .   \tag*{\qed}
\end{align*}

\subsection{Proof of Lemma~\ref{lemma:landmark-representation-error}}
\label{section:proof-landmark-representation-error}

  We partition the $m$ coordinates of the embedding into blocks of $n := \lfloor m / \log_2(1/\delta) \rfloor$ coordinates per block.
  We first consider the part of the embedding corresponding to the first block, say, $\varphi_{1:n}^\star \colon \Xcal \to \RR^n$.
  Define the weight vector $v \in \RR^{n}$ by
\begin{align*}
  v \ = \ v(Z_1,\dotsc,Z_n) & \ := \ \frac{1}{n} (\EE[Y_1 \mid Z_1],\dotsc,\EE[Y_n \mid Z_{n}]) .
\end{align*}
Define $A_i(x) = \EE[Y_i \mid Z_i] g^\star(x,Z_i) - \mu(x)$ for all $x \in \Xcal$, so $A_1(x),\dotsc,A_n(x)$ are i.i.d.~mean-zero random variables.
This implies
\begin{align*}
  \EE[(v^\top \varphi_{1:n}^\star(x) -  \mu(x))^2]
  \ & = \ \EE\left[ \left( \frac1n \sum_{i=1}^n A_i(x) \right)^2 \right]
  \ = \ \frac{\EE[A_1(x)^2]}{n} .
\end{align*}
Now replacing $x$ with $X$ and taking expectations gives
\begin{align*}
  \EE[(v^\top \varphi_{1:n}^\star(X) -  \mu(X))^2]
  \ & = \ \frac{\EE[A_1(X)^2]}{n}
  \ = \ \frac{\var(\EE[Y_1 \mid Z_1]g^\star(X,Z_1))}{n} .
\end{align*}
By Markov's inequality, the event
\[ \EE[(v^\top \varphi_{1:n}^\star(X) -  \mu(X))^2 \mid Z_1,\dotsc,Z_m] \ \leq \ \frac{2}{n} \var(\EE[Y_1|Z_1]g^\star(X,Z_1)). \]
has probability at least $1/2$.
We can analogously define such a ``good'' event for each block of coordinates.
With probability at least $1-\delta$, at least one of these good events occurs; in this event, we can pick any such ``good'' block, set the corresponding weights in $w$ according to the construction above, and set the remaining weights in $w$ to zero.
This produces the desired guarantee.
\qed

\subsection{Proof of Proposition~\ref{prop:tm-landmark}}
\label{section:proof-tm-landmark}

We first make a simple observation about the relationship between terms in the definitions of $\veps_{\landmark}$ and $\veps_{\direct}$.
\begin{proposition}
  \label{prop:jensen}
  If $X \independent Z \mid H$, then for any $(x,z) \in \Xcal \times \Zcal$,
  \begin{align*}
    g^\star(x,z)^2
    & \ \leq \
    \EE\left[ \left( \frac{p_{X \mid H}(x \mid H) p_{Z \mid H}(z \mid H)}{p_X(x) p_Z(z)} \right)^2 \right ]
  \end{align*}
  where $H \sim p_H$.
\end{proposition}
\begin{proof}
  Fix any $(x,z) \in \Xcal \times \Zcal$.
  Then
  \begin{align*}
    g^\star(x,z)^2
    & \ = \ \left( \frac{p_{X,Z}(x,z)}{p_X(x)p_Z(z)} \right)^2 \\
    & \ = \ \frac1{(p_X(x)p_Z(z))^2} \left( \int p_{X,Z,H}(x,z,h) \dif h \right)^2 \\
    & \ = \ \frac1{(p_X(x)p_Z(z))^2} \left( \int p_{X\mid H}(x \mid h) p_{Z\mid H}(z \mid h) p_H(h) \dif h \right)^2 \\
    & \ \leq \ \frac1{(p_X(x)p_Z(z))^2} \int \left( p_{X\mid H}(x \mid h) p_{Z\mid H}(z \mid h) \right)^2 p_H(h) \dif h \\
    & \ = \ \EE\left[ \left( \frac{p_{X\mid H}(x \mid H) p_{Z\mid H}(z \mid H)}{p_X(x)p_Z(z)} \right)^2 \right] ,
  \end{align*}
  where the inequality follows from Jensen's inequality.
\end{proof}

Now we return to the proof of Proposition~\ref{prop:tm-landmark}.
Since $Y$ takes values in $[-1,1]$, it suffices to bound $\EE[g^\star(X,Z_1)^2]$.
By Proposition~\ref{prop:jensen}, for any $(x,z_1)$ we have
\[
  g^\star(x,z_1) \ \leq \ 
  \EE \left[ \left( \frac{p_{X \mid \Theta}(x \mid \bar\Theta) p_{Z \mid \Theta}(z_1 \mid \bar\Theta) }{p_X(x) p_Z(z_1)} \right)^2 \right].
\]
Now replacing $(x,z_1)$ with $(X,Z_1)$ and taking expectations on both sides, the result follows by Proposition~\ref{prop:tm}.
\qed

We note that a direct analysis of $\EE[g^\star(X,Z_1)^2]$ is also straightforward, and would ultimately involve only second-moments of $\bar\Theta$ (as opposed to fourth-moments, considered in Proposition~\ref{prop:tm}).
However, the final bound is the same.

\subsection{Proof of Proposition~\ref{prop:gaussian-landmark}}
\label{section:proof-gaussian-landmark}

The marginal distribution of $(X,Z)$ is $\Ncal(0,\Sigma)$, where
\[ \Sigma \ :=  \ \begin{pmatrix}
1+\sigma^2 & \sigma^2 \\
\sigma^2 & 1+\sigma^2
\end{pmatrix}.  \]
Thus, letting $v := (x,z)$, we have
\begin{align*}
   p_{X,Z}(v) \ &= \ \frac{1}{2\pi \sqrt{1+2\sigma^2}} \exp \left( -\frac{1}{2} v^\top \Sigma^{-1}v  \right), \\
p_X(\cdot) = p_Z(\cdot) \ &= \ \frac{1}{\sqrt{2\pi(1+\sigma^2)}} \exp\left(-\frac{(\cdot)^2}{2(1+\sigma^2)}\right)
\end{align*}
Therefore, for any fixed $(x,z) \in \RR^2$,
\begin{align*}
  g^\star(x,z)^2
  = \frac{p_{X,Z}(x,z)^2}{p_X(x)^2p_Z(z)^2}
  = \frac{(1+\sigma^2)^2}{1+2\sigma^2} \exp\left( - v^\top \Sigma^{-1} v + \frac{x^2+z^2}{1+\sigma^2} \right) 
  = \frac{(1+\sigma^2)^2}{1+2\sigma^2} \exp\left( v^\top A v \right)
\end{align*}
where
\[
  A \ := \
  -\frac{\sigma^2}{1+2\sigma^2}
  \begin{pmatrix}
    \frac{\sigma^2}{1+\sigma^2} & -1 \\
    -1 & \frac{\sigma^2}{1+\sigma^2}
  \end{pmatrix}
  \ = \ \lambda_1 u_1u_1^\top + \lambda_2 u_2u_2^\top
\]
has eigenvalues $\lambda_1 = \frac{\sigma^2}{(1+\sigma^2)(1+2\sigma^2)}$ and $\lambda_2 = -\frac{\sigma^2}{1+\sigma^2}$ corresponding to some orthonormal eigenvectors $u_1$ and $u_2$.
Replacing $(x,z)$ with $V := (X,\tilde Z) \sim p_X \otimes p_Z$ and taking expectation gives
\begin{align*}
  \EE[ g^\star(X,\tilde Z)^2 ]
  & \ = \ \frac{(1+\sigma^2)^2}{1+2\sigma^2} \EE\left[ \exp\left( \lambda_1 (u_1^\top V)^2 + \lambda_2 (u_2^\top V)^2 \right) \right] \\
  & \ = \ \frac{(1+\sigma^2)^2}{1+2\sigma^2} \EE\left[ \exp\left( \lambda_1 (u_1^\top V)^2 \right) \right] \EE\left[ \exp\left( \lambda_2 (u_2^\top V)^2 \right) \right] \\
  & \ = \ \frac{(1+\sigma^2)^2}{1+2\sigma^2} \cdot \frac1{\sqrt{1-2\lambda_1(1+\sigma^2)}} \cdot \frac1{\sqrt{1-2\lambda_2(1+\sigma^2)}} \\
  & \ = \ \frac{(1+\sigma^2)^2}{1+2\sigma^2} \cdot \frac1{\sqrt{1-\frac{2\sigma^2}{1+2\sigma^2}}} \cdot \frac1{\sqrt{1+2\sigma^2}} \\
  & \ = \ \frac{(1+\sigma^2)^2}{1+2\sigma^2} .
\end{align*}
Above, we use the fact that $u_1^\top V$ and $u_2^\top V$ are
independent $\Ncal(0,1+\sigma^2)$ random variables so that the
expectations in the second line are the chi-squared moment generating
function. 
 \qed

\subsection{Proof of Lemma~\ref{lemma:direct-embedding-approximation}}
\label{section:proof-direct-embedding-approximation}

We take $w$ to be
\[ w \ := \ \EE\left[ \EE[Y \mid Z] \psi(Z) \right] . \]
  Let $(\tilde X, \tilde Z, \tilde Y)$ be an independent copy of $(X,Z,Y)$, and observe that
\begin{align*}
  \mu(X)
  & \ = \ \EE\left[ \EE[ Y \mid Z ] \mid X \right]
  \ = \ \EE\left[ \EE[ \tilde Y \mid \tilde Z ] g^\star(X,\tilde Z) \mid X \right] .
\end{align*}
Therefore
\begin{align*}
  \EE[ (w^\top \eta(X) - \mu(X))^2 ]
  & \ = \ \EE\left[ \left( \EE\left[ \EE[ \tilde Y \mid \tilde Z ] \psi(\tilde Z) \right]^\top \eta(X) - \mu(X) \right)^2 \right] \\
  & \ = \ \EE\left[ \left( \EE\left[ \EE[ \tilde Y \mid \tilde Z ] \psi(\tilde Z)^\top \eta(X) - \EE[\tilde Y \mid \tilde Z] g^\star(X,\tilde Z) \mid X \right] \right)^2 \right] \\
  & \ \leq \ \EE\left[ \EE\left[ \EE[ \tilde Y \mid \tilde Z ]^2 \mid X \right] \cdot \EE\left[ \left( \psi(\tilde Z)^\top \eta(X) - g^\star(X,\tilde Z) \right)^2 \mid X \right] \right] \\
  & \ = \ \EE\left[ \EE\left[ \EE[ \tilde Y \mid \tilde Z ]^2 \right] \cdot \EE\left[ \left( \psi(\tilde Z)^\top \eta(X) - g^\star(X,\tilde Z) \right)^2 \mid X \right] \right] \\
  & \ \leq \ \EE\left[ \EE[ \tilde Y^2 ] \cdot \EE\left[ \left( \psi(\tilde Z)^\top \eta(X) - g^\star(X,\tilde Z) \right)^2 \mid X \right] \right] \\
  & \ = \ \EE[ Y^2 ] \cdot \EE\left[ \left( \psi(\tilde Z)^\top \eta(X) - g^\star(X,\tilde Z) \right)^2 \right] .
\end{align*}
Above, the first inequality follows from Cauchy-Schwarz, the subsequent equality uses the independence of $(\tilde Z, \tilde Y)$ and $X$, and the second the inequality follows from Jensen's inequality and the law of total expectation. \qed

\subsection{Proof of Lemma~\ref{lemma:general-direct-embedding}}
\label{section:proof-general-direct-embedding}

We will prove this using the probabilistic method, constructing a random embedding of dimension $m$ that satisfies the lemma in expectation.
This will suffice to show that there exists such an embedding. 

Let $H_1,\dotsc,H_m$ be i.i.d.~copies of $H$.
Define, for each $(x,z) \in \Xcal \times \Zcal$,
\begin{align*}
  \eta(x) \ &= \ \frac{1}{p_X(x)\sqrt{m}} (p_{X\mid H}(x \mid H_i))_{i=1}^m \\
  \psi(z) \ &= \ \frac{1}{p_Z(z)\sqrt{m}} (p_{Z\mid H}(z \mid H_i))_{i=1}^m \\
  B_i(x,z) \ &= \ \frac{p_{X\mid H}(x \mid H_i)p_{Z\mid H}(z \mid H_i)}{p_X(x)p_Z(z)} - g^\star(x,z) .
\end{align*}
Observe that $B_1(x,z),\dotsc,B_m(x,z)$ are i.i.d.~mean-zero random variables, and
\[ \eta(x)^\top \psi(z) - g^\star(x,z) \ = \ \frac1m \sum_{i=1}^m B_i(x,z) , \]
and
\[ \EE[ (\eta(x)^\top \psi(z) - g^\star(x,z))^2 ] \ = \ \frac{\EE[B_1(x,z)^2]}{m} . \]
Now replacing $(x,z)$ with $(X,\tilde Z) \sim p_X \otimes p_Z$ and taking expectations gives
\begin{align*}
  \EE\left[ (\eta(X)^\top \psi(\tilde Z) - g^\star(X,\tilde Z))^2 \right]
  & \ = \ \frac{\EE[B_1(X,\tilde Z)^2]}{m} \ = \ \frac1m \var\left( \frac{p_{X\mid H}(X \mid \bar H) p_{Z \mid H}(\tilde Z \mid \bar H)}{p_X(X) p_Z(\tilde Z)} \right) .  \tag*{\qed}
\end{align*}

\subsection{Proof of Proposition~\ref{prop:tm-prob}}
\label{section:proof-tm-prob}

By Lemma~\ref{lemma:general-direct-embedding}, it suffices to bound
\[
  \var\left( \frac{p_{X \mid \Theta}(X \mid \bar\Theta) p_{Z \mid \Theta}(\tilde Z \mid \bar\Theta) }{p_X(X) p_Z(\tilde Z)} \right) ,
\]
for $(X,\tilde Z,\bar\Theta) \sim p_X \otimes p_Z \otimes p_\Theta$.
This, in turn, is bounded above by
\[
  \EE\left[ \left( \frac{p_{X \mid \Theta}(X \mid \bar\Theta) p_{Z \mid \Theta}(\tilde Z \mid \bar\Theta) }{p_X(X) p_Z(\tilde Z)} \right)^2 \right] .
\]
So Proposition~\ref{prop:tm-prob} follows immediately from the following result.

\begin{proposition}
  \label{prop:tm}
  In the topic modeling setting,
  \[
    \EE\left[ \left( \frac{p_{X \mid \Theta}(X \mid \bar\Theta) p_{Z \mid \Theta}(\tilde Z \mid \bar\Theta) }{p_X(X) p_Z(\tilde Z)} \right)^2 \right]
    \ = \
    \begin{cases}
      \Theta(1)
      & \text{if $\alpha = \Theta(1)$ as $K\to\infty$} ; \\
      \Theta(K^2)
      & \text{if $\alpha \leq 1/K$} .
    \end{cases}
  \]
\end{proposition}
\begin{proof}
  For any word $v \in \Vcal$, let $k(v) \in [K]$ denote the unique topic for which $P_k(v) > 0$.
  For any $v \in \Vcal$ and $\theta \in \Delta^{K-1}$, we have
  \begin{align*}
    p_X(v) \ &= \ \frac{1}{K} P_{k(v)}(v) , \\
    p_{X \mid \Theta}(v \mid \theta) \ &= \ \theta_{k(v)} P_{k(v)}(v) .
  \end{align*}
  Therefore, we have for any $\theta \in \Delta^{K-1}$, $x \in \Vcal$, and $z \in \Vcal$,
  \begin{align*}
    \left( \frac{p_{X \mid \Theta}(x \mid \theta) p_{Z \mid \Theta}(z \mid \theta) }{p_X(x) p_Z(z)} \right)^2 \ &= \ K^4 \cdot \theta_{k(x)}^2 \theta_{k(z)}^2.
  \end{align*}
  Replacing $(x,z,\theta)$ with $(X,\tilde Z,\bar\Theta)$ and taking expecatations gives
  \begin{align*}
    \EE\left[ \left( \frac{p_{X \mid \Theta}(X \mid \bar\Theta) p_{Z \mid \Theta}(\tilde Z \mid \bar\Theta) }{p_X(X) p_Z(\tilde Z)} \right)^2 \right]
    \ &= \ \EE\left[ K^4 \cdot \bar\Theta_{k(X)}^2 \bar\Theta_{k(\tilde Z)}^2 \right]  \\
    \ &= \ K^4 \cdot \sum_{k=1}^K \sum_{k'=1}^K \Pr(k(X) = k) \Pr(k(\tilde Z) = k') \EE\left[ \bar\Theta_{k}^2 \bar\Theta_{k'}^2 \right]  \\
    \ &= \ K^2 \cdot \sum_{k=1}^K \sum_{k'=1}^K \EE\left[ \bar\Theta_{k}^2 \bar\Theta_{k'}^2 \right]  \\
    \ &= \ K^3 \cdot \left( \EE[ \bar\Theta_1^4 ] + (K-1) \EE[ \bar\Theta_1^2\bar\Theta_2^2 ]
    \right)
  \end{align*}
  where the fourth and fifth steps follow by symmetry.
  The fourth-moments in the final expression are:
  \begin{align*}
  \EE[\bar\Theta_1^2 \bar\Theta_2^2] \ &= \ \frac{\Gamma(K \alpha)}{\Gamma(K\alpha + 4)} \cdot \left( \frac{\Gamma(\alpha+2)}{\Gamma(\alpha)}\right)^2 \\
  \EE[\bar\Theta_1^4] \ &= \ \frac{\Gamma(K \alpha)}{\Gamma(K\alpha + 4)} \cdot  \frac{\Gamma(\alpha+4)}{\Gamma(\alpha)} .
  \end{align*}
  Therefore, we have the following:
  \begin{enumerate}
    \item For $\alpha = \Theta(1)$ and $K \to \infty$,
      \begin{align*}
      \EE\left[ \left( \frac{p_{X \mid \Theta}(X \mid \bar\Theta) p_{Z \mid \Theta}(\tilde Z \mid \bar\Theta) }{p_X(X) p_Z(\tilde Z)} \right)^2 \right]
      & \ = \ K^3 \left( \Theta \left( \frac{1}{K^4} \right) +  \Theta\left( \frac{1}{K^3} \right) \right)
      \  = \ \Theta( 1 ) .
      \end{align*}
    \item For $\alpha \leq 1/K$,
      \begin{align*}
      \EE\left[ \left( \frac{p_{X \mid \Theta}(X \mid \bar\Theta) p_{Z \mid \Theta}(\tilde Z \mid \bar\Theta) }{p_X(X) p_Z(\tilde Z)} \right)^2 \right]
      & \ = \ K^3 \left( \Theta \left( \frac{1}{K} \right) +  \Theta\left( \alpha \right) \right)
      \  = \ \Theta( K^2 ) .
      \end{align*}
  \end{enumerate}
\end{proof}

\subsection{Proof of Proposition~\ref{prop:gaussian-prob}}
\label{section:proof-gaussian-prob}

The proof is similar to that of Proposition~\ref{prop:gaussian-landmark}.
Using similar computations, we obtain
for any fixed $(x,z) \in \RR^2$,
\begin{align*}
  \EE\left[ \left(\frac{p_{X \mid H}(x \mid H)p_{Z \mid H}(z \mid H)}{p_X(x)p_Z(z)} \right)^2 \right]
  & \ = \ \frac{(1+\sigma^2)^2}{\sqrt{1+4\sigma^2}} \exp \left( \frac{2 \sigma^2 (x + z)^2}{1+4\sigma^2} - \frac{\sigma^2(x^2 + z^2)}{1+ \sigma^2} \right) \\
  & \ = \ \frac{(1+\sigma^2)^2}{\sqrt{4\sigma^2 + 1}} \exp\left( v^\top A v \right)
\end{align*}
where
\[
  A \ := \ \frac{\sigma^2}{1 + 4\sigma^2} 
  \begin{pmatrix}
    \frac{1-2\sigma^2}{1+\sigma^2} & 2 \\
    2  & \frac{1-2\sigma^2}{1+\sigma^2}
  \end{pmatrix}
  \ = \ \lambda_1 u_1u_1^\top + \lambda_2 u_2u_2^\top
\]
has eigenvalues $\lambda_1 = \tfrac{3\sigma^2}{(1+\sigma^2)(1+4\sigma^2)}$ and $\lambda_2 = -\tfrac{\sigma^2}{1+\sigma^2}$ corresponding to some orthonormal eigenvectors $u_1$ and $u_2$.
Now replacing $(x,z)$ with $V := (X,\tilde Z) \sim p_X \otimes p_Z$ and taking expectation gives, for $(X,\tilde Z,\bar H) \sim p_X \otimes p_Z \otimes p_H$,
\begin{align*}
  \EE\left[ \left(\frac{p_{X \mid H}(X \mid \bar H)p_{Z \mid H}(\tilde Z \mid \bar H)}{p_X(X)p_Z(\tilde Z)} \right)^2 \right]
  & \ = \ \frac{(1+\sigma^2)^2}{\sqrt{4\sigma^2 + 1}} \EE\left[ \exp\left( \frac{3\sigma^2 (u_1^\top V)^2}{(1+\sigma^2)(1+4\sigma^2)} - \frac{\sigma^2 (u_2^\top V)^2}{1+\sigma^2} \right) \right] .
\end{align*}
Since $u_1^\top V$ and $u_2^\top V$ are independent $\Ncal(0,1+\sigma^2)$ random variables, this expression simplifies to
\begin{align*}
  \EE\left[ \left(\frac{p_{X \mid H}(X \mid \bar H)p_{Z \mid H}(\tilde Z \mid \bar H)}{p_X(X)p_Z(\tilde Z)} \right)^2 \right]
  & \ = \
  \frac{(1+\sigma^2)^2}{\sqrt{4\sigma^2 + 1}} \cdot \frac1{\sqrt{1 - \frac{6\sigma^2}{1+4\sigma^2}}} \cdot \frac1{\sqrt{1+2\sigma^2}}
  \ = \ \frac{(1+\sigma^2)^2}{\sqrt{1-4\sigma^4}} .
\end{align*}
The condition $\sigma^2<1/2$ is used to ensure that the expectation in the last equation display is finite.
\qed

\subsection{Proof of Theorem~\ref{thm:error-landmark-embedding}}
\label{section:error-landmark-embedding}

  Let $(X_1,Z_1,Y_1),\dotsc,(X_m,Z_m,Y_m),(X,Z,Y)$ be i.i.d., and let $\varphi^\star(x)_i := g^\star(x,Z_i)$ for all $i=1,\dotsc,m$.
  Let $n := \lfloor m/\log_2(1/\delta) \rfloor$.
  We shall adopt the same block repetition argument as in the proof of Lemma~\ref{lemma:landmark-representation-error}, where the $m$ coordinates are partitioned into groups of $n$ coordinates each.

  We first analyze what happens in the first block of coordinates.
  From the arguments in Lemma~\ref{lemma:landmark-representation-error}, we know with probability $\geq 3/4$, 
  \begin{align}
  \label{eqn:rep-error}
    \EE \left[ \left( \frac{1}{n} \sum_{i=1}^n \EE[Y_i \mid Z_i] g^\star(X,Z_i) -  \mu(X) \right)^2 \mid Z_1,\dotsc,Z_m \right]
    & \ \leq \ \frac{4\var(\EE[Y_1 \mid Z_1]g^\star(X,Z_1))}{n} 
    \ = \ 2 \veps_{\landmark} .
  \end{align}
  We also claim that, with probability $\geq 3/4$,
  \begin{align}
  \label{eqn:opt-error}
  \EE\left[ \frac1n \sum_{i=1}^n \left(g(X, Z_i) - g^\star(X, Z_i)  \right)^2 \mid Z_1,\dotsc,Z_m \right]
  \ \leq \ 4(1+g_{\max})^4 \veps_{\opt,\landmark}.
  \end{align}
To see this, we make the following definitions
\begin{align*}
p^\star(x,z) \ := \  \frac{g^\star(x,z)}{1 + g^\star(x,z)}, \quad \text{and} \quad
p(x,z) \ := \  \frac{g(x,z)}{1 +g(x,z)}.
\end{align*}
Recall that $(X_c,Z_c,Y_c) \sim \Dcontrast$.
Now we have
\begin{align*}
\veps_{\opt,\landmark} \ &= \ L_{\landmark}\left( f \right)  - \inf_{f^\star:\Xcal \times \Xcal \rightarrow \RR} L_{\landmark}\left( \log \circ g^\star \right) \\
\ &= \  \EE \left[ Y_c \log \left(\frac{p^\star(X_c,Z_c)}{p(X_c,Z_c)} \right) + (1-Y_c)  \log \left(\frac{1-p^\star(X_c,Z_c)}{1-p(X_c,Z_c)} \right)  \right] \\
\ &= \ \EE \left[ p^\star(X_c,Z_c) \log \left(\frac{p^\star(X_c,Z_c)}{p(X_c,Z_c)} \right) + (1-p^\star(X_c,Z_c))  \log \left(\frac{1-p^\star(X_c,Z_c)}{1-p(X_c,Z_c)} \right)  \right] \\
\ &= \ \EE \left[ \KL(p^\star(X_c,Z_c), p(X_c,Z_c))  \right]
\end{align*}
where the second-to-last line follows from the fact that $g^\star$ is the odds ratio for the contrastive learning problem and $\KL(p,q)$ denotes the KL divergence between two Bernoulli random variables. Pinsker's inequality tells us that, for any $(x,z) \in \Xcal \times \Zcal$,
\[ \KL(p^\star(x,z), p(x,z)) \ \geq  \ 2 (p^\star(x,z) - p(x,z))^2 \ \geq \ \frac{2}{(1+g_{\max})^4} (g^\star(x,z) -  g(x,z))^2 \]
Since $(X_c,Z_c) \sim \tfrac12 p_{X,Z} + \tfrac12 p_X \otimes p_Z$,
\begin{align*}
\EE\left[  (g^\star(X_c,Z_c) -  g(X_c,Z_c))^2 \right] 
\ &= \ \frac{1}{2} \EE\left[  (g^\star(X,Z) - g(X,Z))^2 \right]  + \frac{1}{2} \EE\left[  (g^\star(X,Z_1) -  g(X,Z_1))^2 \right] \\
\ &\geq \ \frac{1}{2} \EE\left[  (g^\star(X,Z_1) -  g(X,Z_1))^2 \right].
\end{align*}
Therefore, we conclude that
\begin{align*}
  \EE\left[  (g^\star(X,Z_1) -  g(X,Z_1))^2 \right]
  & \ \leq \ (1+g_{\max})^4 \veps_{\opt,\landmark} ,
\end{align*}
and hence also
\begin{align*}
  \EE\left[ \frac1n \sum_{i=1}^n (g^\star(X,Z_i) -  g(X,Z_i))^2 \right]
  & \ \leq \ (1+g_{\max})^4 \veps_{\opt,\landmark} .
\end{align*}
By Markov's inequality, \eqref{eqn:opt-error} holds with probability $3/4$.
A union bound grants that \eqref{eqn:rep-error} and \eqref{eqn:opt-error} hold simultaneously with probability $\geq 1/2$.
Call this the ``good'' event for this first block of landmarks.

Now considering all blocks, with probability $1-\delta$, the good event holds for at least one group of coordinates.
As in the proof of Lemma~\ref{lemma:landmark-representation-error}, we will set $w_i = \frac{1}{n}\EE[Y_i \mid Z_i]$ for the coordinates in the good group and we set $w_i =0$ for all other coordinates.
Thus, with probability $1-\delta$ we can conclude two facts. First, that $\varphi^\star$ satisfies
\[ \EE\left[ \left( w^\top \varphi^\star(X) - \mu(X) \mid Z_1,\dotsc,Z_m \right)^2\right] \ \leq \ 2 \veps_{\landmark}. \]
Second, there is some block of $n$ coordinates (which we take to be $\{1,\dotsc,n\}$ without loss of generality) such that
\begin{align*}
  \EE \left[  \left( w^\top \varphi(X) - w^\top \varphi^\star(X)  \right)^2 \mid Z_1,\dotsc,Z_m \right] 
  & \ \leq \|w\|_2^2 \cdot \EE\left[ \|\varphi(X) - \varphi^*(X)\|_2^2 \mid Z_1,\dotsc,Z_m \right] \\
  & \ \leq \frac1n \EE\left[ \|\varphi(X) - \varphi^*(X)\|_2^2 \mid Z_1,\dotsc,Z_m \right] \\
  & \ = \EE\left[ \frac1n \sum_{i=1}^n \left(g(X, Z_i) - g^\star(X, Z_i)  \right)^2 \mid Z_1,\dotsc,Z_m \right] \\
\ & \ \leq \ 4(1+g_{\max})^4 \veps_{\opt,\landmark}
\end{align*}
where the first inequality follows from Cauchy-Schwarz, the second inequality comes from the boundedness of $Y$, and the third inequality is \eqref{eqn:opt-error}.
Putting it all together with the AM/GM inequality gives us the theorem statement.
\qedhere

\subsection{Proof of Theorem~\ref{thm:error-direct-embedding-realizable}}
\label{section:error-direct-embedding-realizable}

From Lemma~\ref{lemma:direct-embedding-approximation},
\begin{align*}
R(\eta) \ = \ \inf_{w \in \RR^m} \EE \left[ \left( w^\top \eta(X) - \mu(X)  \right)^2 \right]
\ \leq \ \EE[Y^2] \, \EE\left[ \left( \eta(X)^\top \psi(Z) - g^\star(X,Z) \right) \right]
\end{align*}
Therefore, we focus on bounding the second factor on the right-hand side.
For the most part, the proof uses similar arguments as in that of Theorem~\ref{thm:error-landmark-embedding}.

Using the definitions
\begin{align*}
p^\star(x,z) \ := \  \frac{g^\star(x,z)}{1 + g^\star(x,z)} \quad \text{and} \quad
p(x,z) \ := \  \frac{\eta(x)^\top \psi(z)}{1 +\eta(x)^\top \psi(z)}
\end{align*}
we have
\begin{align*}
\veps_{\opt,\direct} & \ = \ L_{\direct}((x,z) \mapsto \eta(x)^\top \psi(z)) - L_{\direct}(g^\star) \\
\ &= \  \EE \left[ Y_c \log \left(\frac{p^\star(X_c,Z_c)}{p(X_c,Z_c)} \right) + (1-Y_c)  \log \left(\frac{1-p^\star(X_c,Z_c)}{1-p(X_c,Z_c)} \right)  \right] \\
\ &= \ \EE \left[ p^\star(X_c,Z_c) \log \left(\frac{p^\star(X_c,Z_c)}{p(X_c,Z_c)} \right) + (1-p^\star(X_c,Z_c))  \log \left(\frac{1-p^\star(X_c,Z_c)}{1-p(X_c,Z_c)} \right)  \right] \\
\ &= \ \EE \left[ \KL(p^\star(X_c,Z_c), p(X_c,Z_c))  \right] .
\end{align*}
By Pinsker's inequality, for any $(x,z) \in \Xcal \times \Zcal$,
\[ \KL(p^\star(x,z), p(x,z)) \ \geq  \ 2 (p^\star(x,z) - p(x,z))^2 \ \geq \ \frac{2}{(1+g_{\max})^4} (g^\star(x,z) -  \eta(x)^\top \psi(z))^2 . \]
Finally, since $(X_c,Z_c) \sim \tfrac12 p_{X,Z} + \tfrac12 p_X \otimes p_Z$,
\begin{align*}
\EE\left[  (g^\star(X_c,Z_c) -  \eta(X_c)^\top \psi(Z_c))^2 \right] 
\ &\geq \ \frac{1}{2} \EE\left[  (g^\star(X,Z) -  \eta(X)^\top \psi(Z))^2 \right] .
\end{align*}
Putting it all together gives us the theorem statement.
\qed

\section{Transfer learning}

Finally, we consider the setting where there is a shift from the distribution $p_{X,Z, Y}$ on which we learned our representations to some new test distribution $q_{X,Z,Y}$. Under what conditions can we guarantee that our representations will transfer gracefully? 

We will consider the scenario where the marginal distributions of $X$ and $Z$ are allowed to change, but the conditional distribution of $Z$ given $X$ remains the same. That is, we will impose the condition 
\[ p_{Z|X}(z \mid x) \ = \ q_{Z|X}(z \mid x)  \] 
for all $x$ and $z$ on $p$ and $q$. To help keep our notation straight, we will use $\EE_{p}[\cdot]$ to denote expectations taken with respect to $p_{X,Z,Y}$ and $\EE_{q}[\cdot]$ to denote expectations taken with respect to $q_{X,Z,Y}$. 

Note that in this setting, the natural analogue of $\mu$ is given by
\[ \mu_q(x) \ = \ \EE_{q} \left[ \EE_{q}[Y \mid Z] \mid X = x \right]. \]
Under redundancy, $\mu_q$ enjoys a similar guarantee as $\mu$.
\begin{lemma}
\label{lemma:mu_X-q-approximation}
Make the definitions
\begin{align*}
\veps^{(q)}_X \ &:= \ \EE_q \left[ \left( \EE_q[Y \mid X] - \EE_q[Y \mid X, Z] \right)^2 \right] \text{ and} \\
\veps^{(q)}_Z \ &:= \ \EE_q \left[ \left( \EE_q[Y \mid Z] - \EE_q[Y \mid X, Z] \right)^2 \right].
\end{align*}
Then we have
$\EE_q[ (\mu_q(X) - \EE_q[Y \mid X,Z])^2 ] \ \leq \ \veps^{(q)}_X + 2\sqrt{\veps^{(q)}_X\veps^{(q)}_Z} + \veps^{(q)}_Z =: \veps^{(q)}_{\mu}. $
\end{lemma}
As the proof of Lemma~\ref{lemma:mu_X-q-approximation} is identical to that of Lemma~\ref{lemma:mu_X-approximation}, we omit it.

Lemma~\ref{lemma:mu_X-q-approximation} tells us that $\mu_q$ is a natural function to approximate. However, given that we solved the contrastive optimization problem on $p$, it is unclear whether or not our representations will transfer gracefully over to approximating $\mu_q$.

Our approach is to `fine-tune' our landmark representation by choosing the landmarks $Z_1, \ldots, Z_m$ according to the target distribution $q_Z(\cdot)$, and embed according to
\[ \varphi^*_q(x) \ := \  \left( g^\star(x, Z_1), \ldots, g^\star(x, Z_m) \right). \]
The following lemma, analogous to Lemma~\ref{lemma:landmark-representation-error}, shows that this does indeed work.
\begin{lemma}
\label{lemma:transfer-representation}
Let $(X_1,Z_1,Y_1),\dotsc,(X_m,Z_m,Y_m),(X,Z,Y)$ be i.i.d. draws from $q_{X,Z,Y}$ and suppose the landmarks used to define $\varphi^*_q$ are $Z_1,\ldots, Z_m$. With probability $1-\delta$, there exists a weight vector $w \in \RR^{m}$ such that
\[ \EE_q[(w^\top \varphi_q^\star(X) -  \mu_q(X))^2 \mid Z_1,\dotsc,Z_m] \ \leq \ \frac{2}{\lfloor m/\log_2(1/\delta) \rfloor}\var\left(\frac{p_Z(Z_1)}{q_Z(Z_1)}\EE_q[Y_1 \mid Z_1]g^\star(X,Z_1) \right) .\]
\end{lemma}
\begin{proof}
We partition the $m$ coordinates of the embedding into blocks of $n := \lfloor m / \log_2(1/\delta) \rfloor$ coordinates per block.
  We first consider the part of the embedding corresponding to the first block, say, $\varphi_{q,1:n}^\star \colon \Xcal \to \RR^n$.
  Define the weight vector $v \in \RR^{n}$ by
\begin{align*}
  v \ = \ v(Z_1,\dotsc,Z_n) & \ := \ \frac{p_Z(Z_i)}{q_Z(Z_i) n} (\EE_q[Y_1 \mid Z_1],\dotsc,\EE_q[Y_n \mid Z_{n}])
\end{align*}
Define $A_i(x) = \frac{p_Z(Z_i)}{q_Z(Z_i)} \EE_q[Y_i \mid Z_i] g^\star(x,Z_i) - \mu(x)$ for all $x \in \Xcal$. Note that for any $x$ we have
\begin{align*}
\EE_{q}\left[ \frac{p_Z(Z_i)}{q_Z(Z_i)} \EE_q[Y_i \mid Z_i] g^\star(x,Z_i)  \right]
\ &= \  \EE_{p} \left[  \EE_{q}[Y \mid	Z = Z_i] g^\star(x, Z_i) \right]  \\
\ &= \  \EE_{p} \left[ \EE_{q}[Y \mid	Z = Z_i] \frac{p_{X,Z}(x,Z_i)}{p_X(x)p_Z(Z_i)} \right]  \\
\ &= \ \EE_{p} \left[ \EE_{q}[Y \mid	Z = Z_i] \frac{p_{Z|X}(Z_i \mid X = x)}{p_Z(Z_i)} \right]  \\
\ &= \ \EE_{q} \left[ \EE_{q}[Y \mid	Z = Z_i] \mid X = x \right]  \\
\ &= \ \mu_q(x)
\end{align*}
where the second-to-last line follows from the assumption that $p_{Z|X}(z \mid x) \ = \ q_{Z|X}(z \mid x)$ for all $x$ and $z$. 
Thus, $A_1(x),\dotsc,A_n(x)$ are i.i.d.~mean-zero random variables.
This implies
\begin{align*}
  \EE_q[(v^\top \varphi_{q,1:n}^\star(x) -  \mu(x))^2]
  \ & = \ \EE_q\left[ \left( \frac1n \sum_{i=1}^n A_i(x) \right)^2 \right]
  \ = \ \frac{\EE_q[A_1(x)^2]}{n} .
\end{align*}
Now replacing $x$ with $X$ and taking expectations gives
\begin{align*}
  \EE_q[(v^\top \varphi_{q,1:n}^\star(X) -  \mu(X))^2]
  \ & = \ \frac{\EE[A_1(X)^2]}{n}
  \ = \ \frac{1}{n} \var\left(\frac{p_Z(Z_1)}{q_Z(Z_1)}\EE_q[Y_1 \mid Z_1]g^\star(X,Z_1) \right)
\end{align*}

By Markov's inequality, the event
\[ \EE_q[(v^\top \varphi_{q,1:n}^\star(X) -  \mu_q(X))^2 \mid Z_1,\dotsc,Z_m] \ \leq \ \frac{2}{n} \var\left(\frac{p_Z(Z_1)}{q_Z(Z_1)}\EE_q[Y_1 \mid Z_1]g^\star(X,Z_1) \right). \]
has probability at least $1/2$.
We can analogously define such a ``good'' event for each block of coordinates.
With probability at least $1-\delta$, at least one of these good events occurs; in this event, we can pick any such ``good'' block, set the corresponding weights in $w$ according to the construction above, and set the remaining weights in $w$ to zero.
This produces the desired guarantee.
\end{proof}

The construction $\varphi^*_q$ is not unique. Indeed, for any distribution $\alpha_Z$ over $Z$, we could sample $Z_1, \ldots, Z_m$ i.i.d. from $\alpha_Z$ and embed according to 
\[ \varphi^*_\alpha(x) \ := \  \left( g^\star(x, Z_1), \ldots, g^\star(x, Z_m) \right). \]
The same arguments of Lemma~\ref{lemma:transfer-representation} apply here as well to give us that with probability $1-\delta$ there exists a $w \in \RR^m$ satisfying
\[ \EE_q[(w^\top \varphi_\alpha^\star(X) -  \mu_q(X))^2 \mid Z_1,\dotsc,Z_m] \ \leq \ \frac{2}{\lfloor m/\log_2(1/\delta) \rfloor}\var\left(\frac{p_Z(Z_1)}{\alpha(Z_1)}\EE_q[Y_1 \mid Z_1]g^\star(X,Z_1) \right)  \]
where the variance is taken with respect to $X \sim q_X$ independently of $Z_1$. 

Given that $\varphi_\alpha^\star$ produces an error bound analogous to the one in Lemma~\ref{lemma:transfer-representation} for any valid distribution $\alpha$, it is natural to ask whether one can find an $\alpha$ whose corresponding variance term is smaller than the variance induced by $q$. As the following lemma shows, there is not much room from improvement over $q$.

\begin{lemma}
Let $\alpha_Z$ be any distribution over $\Zcal$. Suppose that $X \sim q_X$, $Z \sim q_Z$, and $\tilde{Z} \sim \alpha_Z$, all independently. Then
\[ \var\left( \frac{p_Z(Z)}{q_Z(Z)} \EE[Y \mid Z] g^\star(X, Z)  \right) \ \leq \ \var\left( \frac{p_Z(\tilde{Z})}{\alpha_Z(\tilde{Z})}\EE[Y \mid \tilde{Z}] g^\star(X, \tilde{Z})  \right) + \var\left( \EE_q[Y| Z]  \right). \]
\end{lemma}
\begin{proof}
To simplify notation, let $f(z) := \EE_q[Y \mid Z]$. Since $f(Z)g^\star(X, Z)$ and $f(\tilde{Z}) g^\star(X, \tilde{Z})$ have the same mean, it suffices to show
\[ \EE \left( f(Z)g^\star(X, Z) \right)^2 \ \leq \ \EE \left(f(\tilde{Z}) g^\star(X, \tilde{Z}) \right)^2 + \var\left( \EE_q[Y| Z]  \right). \]
Writing out expectations in integral form, we have
\begin{align*}
 &\EE \left(f(\tilde{Z}) \frac{p_Z(\tilde{Z})}{\alpha_Z(\tilde{Z})} g^\star(X, \tilde{Z}) \right)^2 -  \EE \left( \frac{p_Z(Z)}{q_Z(Z)}  f(Z)g^\star(X, Z) \right)^2 \\
 \ &\hspace{7em}= \ \int \int   f(z)^2 p_Z(z)^2 g^\star(x, z)^2 q_X(x) \left( \frac{\alpha_Z(z)}{\alpha_Z(z)^2} - \frac{q_Z(z)}{q_Z(z)^2}\right) \dif x \dif z \\
\ &\hspace{7em}= \  \int \int  f(z)^2 \frac{p_{X,Z}(x,z)^2}{p_X(x)^2} q_X(x) \left( \frac{1}{\alpha_Z(z)} - \frac{1}{q_Z(z)}\right) \dif x \dif z \\
\ &\hspace{7em}= \  \int  f(z)^2 \left( \frac{1}{\alpha_Z(z)} - \frac{1}{q_Z(z)}\right) \left( \int p_{Z|X}(z \mid x )^2 q_X(x) \dif x \right) \dif z \\
\ &\hspace{7em} \geq \ \int  f(z)^2 \left( \frac{1}{\alpha_Z(z)} - \frac{1}{q_Z(z)}\right) \left( \int p_{Z|X}(z \mid x ) q_X(x) \dif x \right)^2 \dif z \\
\ &\hspace{7em} = \ \int  f(z)^2 \left( \frac{1}{\alpha_Z(z)} - \frac{1}{q_Z(z)}\right) q_Z(z)^2 \dif z \\
\ &\hspace{7em}= \ \int  f(z)^2 \frac{q_Z(z)^2}{\alpha_Z(z)} \dif z - \int f(z)^2 q_Z(z) \dif z \\
\ &\hspace{7em} \geq \  \frac{ \left(\int f(z) q_Z(z) \dif z \right)^2}{\int \alpha_Z(z) \dif z } - \int f(z)^2 q_Z(z) \dif z \\
\ &\hspace{7em}= \ \left(\int f(z) q_Z(z) \dif z \right)^2 - \int f(z)^2 q_Z(z) \dif z \\
\ &\hspace{7em}= \ - \var\left( f(Z)  \right)  \ = \ - \var\left( \EE_q[Y| Z]  \right).
\end{align*}
In the above, the first inequality follows from Jensen's inequality, and the second inequality is Titu's lemma (a simple corollary of Cauchy-Schwarz).
\end{proof}

\end{document}